\documentclass[10pt]{article}

 \usepackage{amsmath}
\usepackage{amsthm}
\usepackage{amssymb}
 
\usepackage{enumitem}
  
\usepackage{float}
\usepackage{tikz}
\usepackage{listings}

\usepackage{mathrsfs}
\DeclareMathAlphabet{\mathpzc}{OT1}{pzc}{m}{it}
 \usepackage[mathscr]{eucal}

 \usepackage{algorithm2e}
\usepackage{tablefootnote}
\usepackage{threeparttable}
\usepackage{booktabs, caption, makecell}

\usepackage{blindtext}
\usepackage{hyperref}
\usepackage{nameref}
\usepackage{graphics}
\newcounter{mylabelcounter}
 \usepackage[authoryear,round,longnamesfirst]{natbib}

\makeatletter
\newcommand{\labelText}[2]{%
#1\refstepcounter{mylabelcounter}%
\immediate\write\@auxout{%
  \string\newlabel{#2}{{1}{\thepage}{{\unexpanded{#1}}}{mylabelcounter.\number\value{mylabelcounter}}{}}%
}%
}

\usepackage{pgfplots}
\usepackage{tikz}
\usepackage[english]{babel}
\usepackage{color}
\usepackage{enumitem}
\usepackage{dsfont}

\usepackage{thmtools}
\usepackage{thm-restate}

\newtheorem{theorem}{Theorem} 
\newtheorem{proposition}[theorem]{Proposition}
\newtheorem{definition}[theorem]{Definition}
\newtheorem{remark}[theorem]{Remark}

\newtheorem{lemma}[theorem]{Lemma}

 \newcommand{\reals}{\mathbb{R}}
\newcommand{\ceil}[1]{\left\lceil #1 \right\rceil}
 
\usepackage{thmtools}
\usepackage{thm-restate}
  
\usepackage{thmtools}
 
\usepackage{mathtools}
 
\usepackage{cleveref} 
 
\DeclareMathOperator*{\argmin}{argmin}

\DeclareMathOperator*{\essup}{ess\,sup} 

 \newcommand{\metric}{\mathtt{d}}

\newcommand{\absv}[1]{\left| #1  \right|}
\newcommand{\tabsv}[1]{| #1  |}

\DeclareMathOperator*{\E}{\mathbb{E}}

\newcommand{\Exp}[1]{\E\left[ #1 \right ]}
\newcommand{\ExpD}[2]{\E_{#1}\left[ #2 \right ]}

\DeclareMathOperator{\Rad}{Rad} 
\DeclareMathOperator{\RadL}{Rad} 
\DeclareMathOperator{\Cover}{Cov} 
\DeclareMathOperator{\SetCover}{SCover}

\DeclareMathOperator{\cX}{\mathcal{X}}

\DeclareMathOperator{\cZ}{\mathcal{Z}}

\DeclareMathOperator{\cA}{\mathcal{A}}
\DeclareMathOperator{\cD}{\mathcal{D}}

\newcommand{\hyp}{\mathcal{H}} 

 \DeclareMathOperator{\RR}{\mathbb{R}}
\DeclareMathOperator{\NN}{\mathbb{N}}

\newcommand{\dimM}[1]{{\rm \overline{dim}}_{\mathcal{M}}\left( #1  \right)}

\newcommand{\dimfM}[1]{{\rm dim}_{\rm{fM}}\left( #1  \right)}
 \newcommand{\twoCover}{\mathcal{U}}
 
\newcommand{\est}{\hat \theta}
\newcommand{\estmodel}{\hat \Theta}
\newcommand{\emprisk}{\hat R}
\newcommand{\loss}{\ell}
\newcommand{\event}{\mathcal{E}}
\newcommand{\ind}[1]{\mathds{1}[#1]}   
 \newcommand{\samplespace}{\mathcal{Z}}

 \newcommand{\ball}[1]{\mathbb{B}\left( #1\right)}

\newcommand{\Conv}[1]{\mathrm{Conv}\left( #1\right)}

\newcommand{\bd}[1]{ \mathrm{bd}\left( #1\right)}

 \newcommand{\half}{\frac{1}{2}}

   \newcommand{\shorttitle}{Generalization Guarantees via Algorithm-dependent Rademacher Complexity}
\title{Generalization Guarantees via Algorithm-dependent Rademacher Complexity}
\usepackage{times}
 \author{
 Sarah Sachs \thanks{University of Amsterdam,
Korteweg-de Vries Institute for Mathematics}\\
  \texttt{s.c.sachs@uva.nl}
 \and 
 Tim van Erven  \footnotemark[1]\\
  \texttt{tim@timvanerven.nl}
   \and
     Liam Hodgkinson
    \thanks{ University of Melbourne,
     School of Mathematics and Statistics}\\
    \texttt{lhodgkinson@unimelb.edu.au}
    \and
    Rajiv Khanna
    \thanks{Purdue University,
    Department of Computer Science}\\
    \texttt{rajivak@purdue.edu}
    \and
    Umut \c{S}im\c{s}ekli
    \thanks{Inria, CNRS, Ecole Normale Supérieure, PSL Research University}\\
    \texttt{umut.simsekli@inria.fr}
 }

\begin{document}

\maketitle

\begin{abstract}%
Algorithm- and data-dependent generalization bounds are required to
explain the generalization behavior of modern machine learning
algorithms. In this context, there exists information theoretic
generalization bounds that involve (various forms of) mutual
information, as well as bounds based on hypothesis set stability. We
propose a conceptually related, but technically distinct complexity
measure to control generalization error, which is the empirical
Rademacher complexity of an algorithm- and data-dependent hypothesis
class. Combining standard properties of Rademacher complexity with the
convenient structure of this class, we are able to (i) obtain novel
bounds based on the \emph{finite fractal dimension}, which (a) extend
previous fractal dimension-type bounds from continuous to finite
hypothesis classes, and (b) avoid a mutual information term that was
required in prior work; (ii) we greatly simplify the proof of a recent
dimension-independent generalization bound for stochastic gradient
descent; and (iii) we easily recover results for VC classes and
compression schemes, similar to approaches based on conditional mutual
information.
\end{abstract}


 \section{Introduction}

The generalization error of a learning algorithm is the gap between its
average loss (empirical risk) on a training sample and its expected loss
(risk) on a fresh data point from the same probability distribution. If
the algorithm selects its parameter estimates $\est$ from a set
$\Theta$, then the classical approach to control generalization error is
to derive deviation bounds that hold uniformly over all $\theta \in
\Theta$ \citep{shalev2014understanding}.
However,~\citet{zhang2021understanding} empirically illustrate that in
modern machine learning settings such as neural networks, such an
approach yields overly pessimistic, and sometimes vacuous error bounds
of limited practical value. Hence it has been made clear that
generalization bounds that would reflect the practical observations
should take into account the effects of the data sample $S^n$ and also
the choice of the learning algorithm (e.g., stochastic gradient
descent).

\citet{russo2016controlling,xu2017information} initiated a fertile line
of research by developing algorithm-dependent generalization bounds with
information-theoretic tools. Let $R(\est)$ and $\emprisk(\est)$ denote
the risk and empirical risk of an algorithm. They show that
the
generalization error is at most
\begin{align}
\Exp{R(\est) - \hat R(\est,S^n)} \lesssim \sqrt{\frac{I(\est;S^n)}{n}}, \label{eqn:mi_bound}
\end{align}
where $I(\theta;S^n)$ denotes the \emph{mutual information} (MI) between
the data sample $S^n$ and the output of the algorithm $\est$. This shows
that a weak statistical dependence between the data sample and the
algorithm output implies better generalization. Recently, by using tools
from rate-distortion theory, the bound \eqref{eqn:mi_bound} was linked
to compression, implying that if the algorithm output is compressible in
some sense, it implies good generalization
\citep{pmlr-v178-sefidgaran22a}, in line with the results of
\citep{arora2018stronger,suzuki2018spectral,suzuki2019compression,barsbey2021heavy}.

Looking at the problem from a different angle,
\citet{simsekli2020hausdorff} take into account
the topological structure of the outputs of the learning algorithm by
using tools from fractal geometry \citep{falconer2004fractal}. More
precisely, let $\Theta_{S^n} \subset \Theta$ denote the full trajectory
of a continuous time version of stochastic gradient descent on the
sample $S^n$. Then they prove
generalization bounds based on  a fractal dimension of
$\Theta_{S^n}$, which were later extended and improved by \citet{birdal2021intrinsic,camuto2021fractal,hodgkinson2022generalization,dupuis2023}. 
 Their bounds are of the following general form\footnote{\citet{hodgkinson2022generalization} also proved an in-expectation version of \eqref{eqn:fd_bound} which involved the weaker $I(\Theta_{S^n}; S^n)$ instead of $I_\infty(\Theta_{S^n}; S^n)$.}: with probability at least $1-\delta$,
 \begin{align}
 \label{eqn:fd_bound}
 \sup_{\theta \in \Theta_{S^n}}\absv{R(\theta) - \hat R(\theta,S^n)} \lesssim \sqrt{\frac{\dim{\Theta_{S^n}} + I_\infty(\Theta_{S^n}; S^n) + \log(1/\delta)}{n}} ,
 \end{align}
 where $I_\infty$ denotes the total MI, which is larger than the regular MI, and $\dim$ denotes some notion of fractal dimension (e.g., the
 Minkowski, Hausdorff, or persistent homology dimension). In addition to
 the statistical dependence between the data sample and the random
 hypothesis set $\Theta_{S^n}$ as measured by $I_\infty$, these bounds
 imply that the worst-case error can be controlled by the fractal
 dimension of $\Theta_{S^n}$. This fractal dimension is linked to
 the statistical or topological properties of the learning algorithm; in
 particular, if the algorithm is a stochastic optimizer such as
 stochastic gradient descent.

 While these bounds help to shed light on modern learning problems from
 different viewpoints, the MI terms that they contain can be troublesome
 for several reasons. First, MI can be infinite, which renders the
 bounds vacuous \citep[Section 5]{Bassily2017LearnersTU}. Secondly,
 while the fractal dimension in \eqref{eqn:fd_bound} can be linked to
 concrete and computable properties of the learning algorithm, the MI term typically cannot be given a topological interpretation, which means
 the bounds as a whole also do not have a fully topological
 interpretation.

In order to address the first shortcoming,
\citet{pmlr-v125-steinke20a} introduced the \emph{conditional mutual
information} (CMI) which in contrast to MI is always finite. They show
that the CMI implies generalization under much weaker assumptions than
MI: for instance, it can be controlled if the learning algorithm is a
compression scheme \citep{Littlestone2003RelatingDC} or under
distributional stability assumptions such as differential privacy
\citep{Bassily_2016,Dwork_2015}. As an alternative,
\citet{pmlr-v178-sefidgaran22a} introduced the notion of \emph{lossy
compressible learners}, which also circumvents the cases where MI can be
infinite. Despite these improvements, it remains unclear how to
relate MI-based bounds to topological concepts.\footnote{We note that
\citet[Corollary 7]{pmlr-v178-sefidgaran22a} link MI to fractal
geometry through rate-distortion theory \citep{kawabata1994rate}.
But their result involves the \emph{marginal} distribution of
$\est$, which has limited practical interest.} Furthermore,
\citet{haghifam2022limitations} recently showed that it is impossible to
obtain minimax rates for gradient descent by using the current
information-theoretic frameworks, and argued that new frameworks need to
be developed.   

\paragraph{Contributions}
In this study, we propose an alternative mathematical framework for analyzing algorithm- and data-dependent hypotheses. We make the following contributions:  
\begin{itemize}[itemsep=0pt,topsep=5pt,leftmargin=*,align=left]
\item We prove a generalization bound with respect to an \emph{algorithm-dependent Rademacher complexity} (ARC), Lemma~\ref{keyThm} in Section~\ref{BoundRad}. Interestingly, our construction is conceptually related
to the conditioning in CMI. It is also technically similar to a special
case of the Rademacher complexity for data-dependent hypothesis sets
introduced by \citet{FosterEtAl2019-HSS}. This special case arises when
their hypothesis sets are instantiated as singletons that contain the
output of a learning algorithm on the sample. For this case, our result
is a refinement of their Theorem~1.\footnote{We note that our setting is
not the main focus of \citet{FosterEtAl2019-HSS}, who mostly consider
stability properties for data-dependent hypothesis classes.} Both of
these relations are discussed in more detail below
Definition~\ref{ARCdefn} in Section~\ref{BoundRad}.

 \item Our main contribution is to demonstrate the flexibility of the ARC. In Section~\ref{applications}, we derive several new generalization results and re-obtain known results.  More precisely, 
\begin{itemize}[itemsep=0pt, topsep=0pt]
\item In Section~\ref{SectionFiniteMinkowskiBound} we link ARC to
fractal dimensions using the tools developed by
\citet{alonso2015hausdorff}. This allows us to extend previous fractal
dimension-type bounds from continuous to finite hypothesis classes
without introducing any mutual information term as in
\eqref{eqn:fd_bound}.
\item For stochastic gradient descent on strongly convex and smooth
losses, we use ARC to obtain a greatly simplified proof of a
dimension-independent generalization bound by
\citet{ParkSimsekliErdogdu2022generalization} (Section~\ref{sec:sgd}).
\item For learning algorithms that are compression schemes or produce output in a
VC class, we show that we can obtain the same generalization properties
as those obtained for CMI by \citet{pmlr-v125-steinke20a} (Sections~\ref{sec:compression} and
\ref{sec:vc}).\footnote{\citet{pmlr-v125-steinke20a} can also obtain guarantees for differentially private algorithms. While we
have not investigated whether such results can also be obtained via ARC,
we suspect that information-theoretic tools might be more natural to
analyze differentially private algorithms.}
\end{itemize}
\end{itemize}
We believe that the proposed framework provides a promising alternative
to information-theoretic approaches and opens up future directions,
which we discuss in Section~\ref{sec:conc}. Some of the proofs are
delegated to the appendix.

\section{Preliminaries}

\paragraph{Setting}

Given a sample $S^n = (Z_1,\ldots,Z_n) \in \samplespace^n$ of
independent, identically distributed (i.i.d.)\ observations with common
distribution $\cD$, and a loss function $\ell : \Theta \times
\samplespace \to \reals$, let
\[
  R(\theta) = \ExpD{Z \sim \cD}{\ell(\theta,Z)}
  \qquad\text{and}
  \qquad
  \emprisk(\theta, S^n) = \frac{1}{n} \sum_{i=1}^n\ell(\theta,Z_i)
\]
denote the risk and the empirical risk, respectively. 
We assume throughout that $\Theta$ is a subset of a complete separable
metric space $(\cX,\metric)$, and that $\ell$ is measurable.
A common
instantiation in supervised learning is that $Z = (X,Y) \in \mathcal{X}
\times \mathcal{Y}$ and $\theta$ indexes a class of hypotheses $\hyp =
\{h_\theta : \mathcal{X} \to \mathcal{Y} \mid \theta \in \Theta\}$. Then
the loss is typically defined as a composite loss function via
$\ell(\theta,z) := \tilde \ell(h_\theta(x),y)$ for some choice of
$\tilde \ell : \mathcal{Y} \times \mathcal{Y}$.
A learning algorithm
$\cA$ is a measurable function that maps a sample $S^n$ to an
estimate $\est = \cA(S^n) \in \Theta$. 
We assume that $\cA$ is a
deterministic algorithm for now; at the start of Section~\ref{extRandomOutputs} we
discuss how our results can be extended to randomized algorithms. \paragraph{Rademacher complexity } A standard approach to control the generalization error relies on the
(empirical) Rademacher complexity of the whole class $\Theta$ of
possible outputs of the algorithm:
\begin{equation}\label{eqn:rademacher-def}
  \RadL(\Theta,S^n) = \frac{1}{n}
  \ExpD{\sigma}{\max_{\theta \in \Theta} \sum_{i=1}^n
  \sigma_i \ell(\theta, Z_i)},
\end{equation}
where the expectation is over Rademacher random variables $\sigma =
(\sigma_1,\ldots,\sigma_n) \in \{-1,+1\}^n$, which are i.i.d.\ with
$\Pr(\sigma_i = -1) = \Pr(\sigma_i = +1) = 1/2$. 
It is well known that the Rademacher complexity can be upper bounded in terms of a covering number. 
\paragraph{Covering numbers} 

For any $C \subset \cX$, we will denote its $\epsilon$-covering number by
$\Cover(C,\metric,\epsilon)$ and the corresponding $\epsilon$-cover by
 $\SetCover(C,\metric,\epsilon)$. The box-covering number refers to the
 special case that
 $\cX = \reals^k$ and $\metric$ is the distance induced by the
 $\ell_\infty$-norm. 
 
\section{Generalization via Algorithm-dependent Rademacher Complexity}
 
\label{BoundRad}

In this section, we refine the standard Rademacher bounds on the
generalization error by measuring the Rademacher complexity not on
$\Theta$ but on a smaller set $\estmodel^n$ that depends on the
algorithm and the data, which is defined as follows: consider two
independent samples $S^n_- = (Z_1^{-1},\ldots,Z_n^{-1})$ and $S^n_+ =
(Z_1^{+1},\ldots,Z_n^{+1})$, and, for any $\sigma \in \{-1,+1\}^n$, let
$S^n_\sigma = (Z_1^{\sigma_1},\ldots,Z_n^{\sigma_n})$ denote a combined
sample where $\sigma_i$ determines whether to take $Z_i^{\sigma_i}$ from
sample $S_-^n$ or from sample $S_+^n$. Then
\[
  \estmodel^n := \big\{\cA(S^n_\sigma) : \sigma \in \{-1,+1\}^n\big\}
  \subset \Theta,
\]
which depends on $S_-^n$ and $S_+^n$ and contains all possible outputs
of the algorithm $\cA$ that can be obtained by combining them with
different choices of $\sigma$.  
 \begin{definition}
\label{ARCdefn}
We define the \emph{Algorithm-dependent Rademacher Complexity} (ARC) as
the Rademacher complexity $\RadL(\estmodel^n, S_+^n)$ of the algorithm-
and data-dependent set $\estmodel^n$.
\end{definition}
In our analysis, $S_-^n$ acts as a \emph{ghost sample}, which is
independent of $S_+^n$. 
This allows us to shrink the effective size of the function class from all possible functions indexed
by the parameters $\Theta$ to a finite class of functions indexed by
$\estmodel^n$ that can be realized by the algorithm by exchanging data
points between $S_-^n$ and $S_+^n$ according to Rademacher variables
$\sigma$. The ARC can therefore be seen as measuring the complexity of
the algorithm when only $\sigma$ is unknown, conditional on the
supersample $(S_-^n,S_+^n)$. This is conceptually similar to the
conditional mutual information of \citet{pmlr-v125-steinke20a}, except
that we use Rademacher complexity where they use mutual information and
therefore we get into a technically fully distinct analysis. There is
also a strong connection to the Rademacher complexity for data-dependent
hypothesis sets introduced by \citet{FosterEtAl2019-HSS}. When their
Theorem~1 is specialized to an algorithm-dependent result for our
setting, it gives a bound in terms of the larger class
  $\bar \Theta^n = \{ \cA(S) : S \subset (S_-^n,S_+^n) \text{ such
  that } |S| = n\}$.\footnote{In the notation of
  \citet{FosterEtAl2019-HSS}, our setting corresponds to the case where
  $\hyp_S = \{\cA(S)\}$, $m=n$ and $U = (S,T)$. Then their $\bar
  \hyp_{U,m} = \hyp_{S,T}$, which, in our notation equals the class of
  hypotheses indexed by $\bar \Theta^n$.}
Since $\estmodel^n \subset \bar \Theta^n$, our result is strictly
better, but it is not evident that the improvement is very large.
For instance, our analysis of SGD in Section~\ref{sec:sgd} will still go
through even with the larger set $\bar \Theta^n$ at the cost of an
additional $\log n$ factor in the bound. See
Remark~\ref{rem:hypstability}.

The following key technical result proved in
Appendix~\ref{app:crc_proof}, shows that the ARC can control the
generalization error in the same way as the classical Rademacher
complexity does for fixed hypothesis classes. To state it, we let
$\essup X = \inf \{a : \Pr(X > a) = 0\}$ denote the \emph{essential
supremum} of any random variable $X$.
 \begin{lemma}[Key technical lemma]
 \label{keyThm}
 The expected generalization error of any (deterministic) algorithm
 $\cA : \cZ^n \to \Theta$ with output $\est = \cA(S^n)$ is bounded by
\begin{equation}\label{eqn:rad-exp}
 \ExpD{S^n}{R(\est) - \emprisk(\est,S^n)} \leq 2 \ExpD{S^n_-,S^n_+}{\RadL(\estmodel^n, S^n_+)}.
\end{equation}
Moreover, if there exist a $b$ and a function $h : \Theta \to \reals$
such that the loss $\loss(\theta,z)$ takes values in the bounded
interval $[h(\theta),h(\theta)+b]$ for all $\theta \in \Theta$ and $z
\in \samplespace$, then, for any $\delta \in (0,1]$,
\begin{equation}\label{eqn:rad-prob}
  R(\est) - \emprisk(\est, S^n)
    \leq 4\,\essup_{S_-,S_+}\,\{\RadL(\estmodel^n,S^n_+)\}
      + b\sqrt{\frac{8 \log(2/\delta)}{n}}
\end{equation}
with probability at least $1-\delta$.
\end{lemma} 

The proofs of both results mimic standard Rademacher complexity bounds
on the generalization error, except that we do not start with the
standard upper bound $R(\est) - \emprisk(\est, S^n) \leq \sup_{\theta
\in \Theta}\{R(\theta) - \emprisk(\theta, S^n)\}$, but we instead
replace $\est$ by the maximum over $\theta \in \estmodel^n$ \emph{after}
symmetrization by the ghost sample $S_-^n$. This allows $\estmodel^n$ to
depend on the algorithm $\cA$ as well as $S^n$ and $S_-^n$. For
notational symmetry, we then denote the original sample $S^n$ by $S_+^n$
in the right-hand side of both results. Although the main idea behind
both proofs is the same, it is not the case that \eqref{eqn:rad-prob}
follows directly from \eqref{eqn:rad-exp}, because $\est$ may be highly
unstable, so we cannot apply McDiarmid's inequality to relate $R(\est) -
\emprisk(\est, S^n)$ to its expectation. We therefore prove both results
separately: the in-expectation proof is a variation on the standard
in-expectation argument, which can be found in e.g.\ Lemma~A.5 of
\citet{BartlettBousquetMendelson2005}. The in-probability derivation is
patterned after the proof of Theorem~4.3 in the textbook by
\citet{AnthonyBartlett1999}.

\subsection{Consequences}

\label{extRandomOutputs}

\paragraph{Randomized algorithms}

Although Lemma~\ref{keyThm} is stated only for deterministic
algorithms, it can also be applied to algorithms that randomize. This is
possible by viewing an algorithm as a function of two arguments
$\cA(S^n,\xi)$, where $S^n$ is the sample and $\xi$ is an independent
random variable (e.g.\ a number of random bits) that provides the
randomness. By applying Lemma~\ref{keyThm} conditional on $\xi$, we
obtain a generalization bound that holds for any value of $\xi$ and
hence also almost surely when $\xi$ is drawn at random. This is the
approach we take to analyze stochastic gradient descent in
Section~\ref{sec:sgd}.

\paragraph{Exploiting standard properties of Rademacher complexity}

We may use all the well-known properties of standard Rademacher
complexity to upper bound the ARC. In particular, throughout the paper
we will repeatedly rely on the following upper bound in terms of the
covering number of $\estmodel^n$, which holds for bounded, Lipschitz
continuous losses:
\begin{proposition}
\label{Lem:CoverBound}
Suppose that
$\ell(\theta,z)$ is $L$-Lipschitz continuous in $\theta$ for any $z$ and
takes values in $[a,a+b]$. Then
\begin{equation}\label{optCoverBound}
  \RadL(\estmodel^n,S^n) \leq L \epsilon + b
  \sqrt{\frac{\log \Cover(\estmodel^n,\metric,\epsilon)}{n}}
  \qquad
  \text{for all $\epsilon > 0$.}
\end{equation}
\end{proposition}
\begin{proof}
 For any $\epsilon > 0$, Lipschitzness of the loss implies that
\[
  \RadL(\estmodel^n,S^n)
    \leq \epsilon L + \RadL(\SetCover(\estmodel^n,\metric,\epsilon),S^n)
    \leq \epsilon L + b \sqrt{\frac{2\ln
    \Cover(\estmodel^n,\metric,\epsilon)}{n}},
\]
where the second inequality follows from Massart's lemma
\citep{shalev2014understanding}.
\end{proof}

\section{Applications: controlling algorithm-dependent Rademacher
complexity}
\label{applications}

In this section, we derive new results and recover known results
by controlling the ARC defined in the previous section. First, in
Section~\ref{SectionFiniteMinkowskiBound}, we provide a new bound with
respect to fractal dimensions. This result allows for control
of the generalization error based on the topological properties of
$\estmodel^n$ or its limiting set $\estmodel$ as $n$ increases. Second, in
Section~\ref{sec:sgd}, we consider the projected stochastic gradient
descent algorithm and recover the results of
\cite{ParkSimsekliErdogdu2022generalization} in a simple way by bounding
the
ARC. Third, in Sections~\ref{sec:compression} and \ref{sec:vc}, similar
to \cite{pmlr-v125-steinke20a}, we show that generalization guarantees
under a compressibility condition and for VC-classes can be easily
obtained via ARC as well. 

\subsection{Finite fractal dimensions}
\label{SectionFiniteMinkowskiBound}

In this section, we provide a bound on the generalization error with respect to a \emph{finite Minkowski dimension}, in the vein of recent results connecting error to fractal geometry (cf.\ \citep{simsekli2020hausdorff,birdal2021intrinsic,hodgkinson2022generalization,dupuis2023}). 
The finite Minkowski dimension was introduced in \citep{alonso2015hausdorff} as an extension of the classical \emph{Minkowski dimension} to finite sets. A comprehensive summary of the formal definitions for the finite Minkowski dimension and the relevant existing results is provided in Appendix \ref{appendixDimFM}. For brevity, here we consider a simplified definition of the finite Minkowski dimension. Under some small additional assumptions which exclude notorious edge cases, these simplified definitions coincide with the original definitions in \citep{alonso2015hausdorff}.
\begin{definition}[Diameters] 
Let $C$ be a finite set in a metric space $(\cX,\metric)$ with $\absv{C}
\geq 2$, and let $\nu_C:C\rightarrow\RR$ map points in $C$ to the
distance to their nearest neighbor, that is, $\nu_C(a) = \min \{ \metric(a,b)
: b \in C\setminus \{a\} \}$ for $a \in C$. The covering diameter
$\nabla (C)$ and the diameter $\Delta(C)$ of $C$ are then
\begin{align*}
\nabla (C) := \max_{a \in C} \nu_C(a)\, ,\qquad \text{ and } \qquad
\Delta(C) := \max\{ \metric(a,b) : a,b \in C\}\,  .
\end{align*}
A set $C$ is called \emph{non-focal} if $\nabla(C) < \Delta(C)$, and \emph{focal} otherwise. 
\end{definition}
For example, the set of vertices of a simplex comprises a focal set, while a set with no equally distant points is non-focal.
Our results will be for non-focal sets only, which rules out pathological edge cases. 
Hence, \emph{throughout this section}, we assume that $\estmodel^n$ is non-focal. 
The definition of the finite Minkowski dimension follows a similar box-counting construction as the standard Minkowski dimension. The main twist which yields nontrivial values for finite sets is to consider covers that contain at least two points.   
\begin{definition}[finite Minkowski dimension]
\label{Definition:fMinDim}
A family of sets $\twoCover = \{U_i\}$ is a 2-cover of a finite set $C$ if
each $U_i \subseteq C$, $|U_i| \geq 2$, and
$C = \bigcup_{U \in \twoCover} U$. For any non-focal set $C$ and parameter $a \geq \nabla(C)$, the covering cardinality is
\begin{align*}
T_a (C) = \min\{ \absv{\twoCover} : \twoCover \text{ is a 2-cover  and } U: \Delta(U) \leq  a\text{ for all } U \in \mathcal{U}\}. 
\end{align*}
For $T_{\nabla(C)}(C)$ we write $T(C)$. 
If $C$ is finite and non-focal with $\absv{C} >2$, then the finite Minkowski dimension of $C$ is 
\[
\dimfM{C} = \frac{\log T(C)}{\log\frac{\Delta(C)}{\nabla(C)}}. 
\]
\end{definition}
Although it may not be obvious from the definition, \citet{alonso2015hausdorff} shows that the finite Minkowski dimension is a natural finite analog of the classical Minkowski dimension (cf.\ Equation \eqref{eqn:dimm} below and \citep{falconer2004fractal}), as the two definitions are consistent under appropriate limits. Most importantly, like the classical Minkowski dimension, if $f$ is H\"{o}lder continuous such that $c_1 \|x - y\|^\beta \leq |f(x)-f(y)|\leq c_2\|x - y\|^\beta$, then the finite Minkowski dimension of $f(C)$ satisfies $\dimfM{f(C)} = \beta \dimfM{C}$. In this sense, the finite Minkowski dimension is well-suited to measure local clustering as in \citep{hodgkinson2022generalization}.
Our primary fractal dimension generalization bound is shown in Theorem \ref{thm:dimfMBoundMain}, which arises by taking $\epsilon = \nabla(\estmodel^n)$ in (\ref{optCoverBound}).

\begin{theorem}
\label{thm:dimfMBoundMain}
If $\ell$ is $L$-Lipschitz in $\theta$ and bounded by $b$, $n > 2$, and $\estmodel^n \subset \reals^k$ is non-focal, then for any set $F \supseteq \estmodel^n$,
\[
  \RadL(\estmodel^n, S_+^n) \leq  \mathcal{D}_n(F) \coloneqq L\nabla(F) + b\sqrt{\frac{\dimfM{F}}{n} \log \frac{\Delta(F)}{\nabla(F)}}.
\]
\end{theorem}

There are a few variants and consequences of Theorem \ref{thm:dimfMBoundMain} worth mentioning.
\begin{itemize}[leftmargin=*]
\item \textbf{Trivial upper bound:} for $\delta > 0$ sufficiently small, consider the set $\estmodel^n_\delta = \estmodel^n \cup (\estmodel^n + \delta \boldsymbol{1})$, which satisfies $\nabla(\estmodel^n_\delta) = \delta$. By \cite[Proposition 5.3]{alonso2015hausdorff}, for any finite set $F$, 
$$\dimfM{F} \leq (\log(|F|) -1)/ \log \frac{\Delta(F)}{\nabla(F)}$$ 
and so $\mathcal{D}_n(\estmodel^n_\delta) \leq L\delta + b\sqrt{\log (2\tabsv{\estmodel^n})/n}$. Taking $\delta \to 0^+$, this implies that 
$$\RadL(\estmodel^n,S_+^n) \leq b\sqrt{\log (2\tabsv{\estmodel^n})/n}.$$
\item \textbf{No outliers:} if $\epsilon^\ast$ minimizing (\ref{optCoverBound}) satisfies $\nabla(\estmodel^n) \leq \epsilon^\ast$, then 
\begin{equation}
  \label{nooutliers}
\RadL(\estmodel^n, S_+^n) \leq 2 b\sqrt{\frac{\dimfM{\estmodel^n}}{n} \log \frac{\Delta(\estmodel^n)}{\nabla(\estmodel^n)}}.
\end{equation}

\item \textbf{Steiner points:} in Appendix \ref{SteinerPoints}, we show that it is possible to construct a set of points $P$ with $\tabsv{P} \leq \tabsv{\estmodel^n}$ (which we refer to as \emph{Steiner points}, following a similar concept in graph theory) such that $\nabla(\estmodel^n \cup P) \leq \epsilon^\ast$ and so the ``no outliers'' case (\ref{nooutliers}) follows for $\estmodel^n \cup P$. This 
 simplifies the bound at the cost of extending the set $\estmodel^n$ by a finite number of additional points.
\end{itemize}

A very interesting simplification compared to Theorem
\ref{thm:dimfMBoundMain} arises in the limit as $n \to \infty$. Let
$\estmodel = \bigcup_{n=1}^{\infty} \estmodel^n$ be the collection of
all output sets of the algorithm obtained for different $n$. Since
$\estmodel$ is infinite, we may now consider its upper Minkowski
dimension, defined by
\begin{align}
\label{eqn:dimm}
\overline{\dim}_{\mathcal{M}}(\estmodel) = \limsup_{\delta \to 0^+}
\frac{\log \Cover(\estmodel, \metric, \delta)}{\log(1/\delta)}.
\end{align}
 
\begin{theorem}
\label{thm:limit}
Suppose $\ell$ is $L$-Lipschitz
continuous in $\theta$ and takes values in an interval of length $b$.
Then
\[
\limsup_{n \to \infty} \frac{\RadL(\estmodel^n, S_+^n)}{\sqrt{\log(n)/n}}
\leq b\sqrt{\frac{\overline{\dim}_\mathcal{M}(\estmodel)}{2}}.
\]
\end{theorem}
The proof is a straightforward consequence of \eqref{optCoverBound} and
the definition of the upper Minkowski dimension; see
Appendix \ref{SteinerPoints}. It would be
possible to adapt the proof to go through with $\estmodel$ replaced by
$\estmodel' = \bigcup_{n=n'}^{\infty} \estmodel^n$ for some finite
integer $n'$, which at first sight looks like it gives a stronger conclusion. On
closer inspection, however, the two results turn out to be equivalent,
since $\estmodel \setminus \estmodel' = \bigcup_{n=1}^{n'-1}
\estmodel^n$ is a finite set, which implies that
$\overline{\dim}_\mathcal{M}(\estmodel') =
\overline{\dim}_\mathcal{M}(\estmodel)$.

Using Fatou's lemma, Lemma \ref{keyThm} and Theorem \ref{thm:limit} together
imply that the expected generalization error of any (deterministic)
algorithm $\cA : \cZ^n \to \Theta$ is
$\mathcal{O}(\sqrt{ \log(n)/n})$ whenever
$\E[\overline{\dim}_{\mathcal{M}}(\estmodel)]$ is finite,
and satisfies 
\begin{equation}
\label{limdimBound}
\limsup_{n\to\infty} \frac{\ExpD{S^n}{R(\est) -
\emprisk(\est,S^n)}}{\sqrt{\log(n) / n}} \leq b\sqrt{2\E[\overline{\dim}_{\mathcal{M}}(\estmodel)]}.
\end{equation}
\citet{hodgkinson2022generalization} and \citet{dupuis2023} both also
obtain $\mathcal{O}(n^{-1/2})$ bounds on the generalization error
involving Minkowski dimensions of algorithm-dependent sets, but they
assume Ahlfors regularity and/or incorporate a potentially vacuous
mutual information term, while ours requires neither.

\subsection{Dimension-independent generalization for SGD}
\label{sec:sgd}
Stochastic gradient descent (SGD) is a randomized algorithm with
iterative updates of the form $\est_{t+1} = \Phi_{i_t}(\est_t)$, where
\[
  \Phi_i(\theta) = \Pi_{\Theta}(\theta - \eta \nabla \ell(\theta,Z_i)),
\]
$\Pi_\Theta(\theta)$ denotes the projection of $\theta$ onto
$\Theta$, and the algorithm depends on the choice of an initialization point
$\theta_1 \in \Theta$, a step size $\eta > 0$, and indices $i_t$ that
are chosen uniformly at random from $\{1,\ldots,n\}$. In the sequel, assume that
$\Theta \subset \reals^k$ is compact and convex, and let the final output
of the algorithm after $T$ updates be $\est = \est_{T+1}$. For the loss
$\ell(\theta,z)$, we assume that it is differentiable in $\theta$, and
satisfies the following conditions for all $z \in \samplespace$:
\begin{itemize}
\item \emph{$\alpha$-Strong convexity}: for any $\theta,\theta' \in
\Theta$, $(\nabla_\theta \ell(\theta,z) - \nabla_{\theta'} \ell(\theta',z)) \cdot (\theta - \theta') \geq \alpha \|\theta - \theta'\|^2$,
\item \emph{$\beta$-Smoothness}: for any $\theta,\theta' \in \Theta$,
$\|\nabla_\theta \ell(\theta,z) - \nabla_{\theta'} \ell(\theta',z)\| \leq \beta \|\theta - \theta'\|$.
\end{itemize}
We further assume:
\begin{itemize}
\item \emph{$L$-Weak Lipschitz continuity}: For $L > 0$, there exists $h : \Theta
\to \reals$ such that, for any $\theta,\theta' \in \Theta$ and any $z
\in \samplespace$, $|\ell(\theta,z) - h(\theta) - (\ell(\theta',z) -
h(\theta'))| \leq L\|\theta - \theta'\|$.
\end{itemize}
Under these assumptions, \citet{ParkSimsekliErdogdu2022generalization}
obtain generalization bounds for SGD that do not depend explicitly on
the ambient dimension $k$. The key step in their analysis is that
$\Phi_i$ is a $\gamma$-contractive operator for $\gamma = \sqrt{1 - 2
\alpha \eta + \alpha \beta \eta^2}$; i.e., $\|\Phi_i(\theta) -
\Phi_i(\theta')\| \leq \gamma \|\theta - \theta'\|$ for any parameters
$\theta, \theta' \in \Theta$. This causes SGD to forget about previous
iterates at a rate that depends on $\gamma$. Using the same idea, we can
recover their first main result, presented as
\citep[Theorem~2.1]{ParkSimsekliErdogdu2022generalization}, with a much
simpler proof based on ARC.  
\begin{theorem}\label{thm:sgd}
If the loss $\ell$ takes values in an interval $[a,a+b]$, is
$\alpha$-strongly convex, $\beta$-smooth and $L$-weakly Lipschitz
continuous, and if $\Theta$ is compact and convex with diameter at most
$\Delta(\Theta) \leq R$, then for any initialization $\theta_1 \in
\Theta$, any $\eta \in (0,2/\beta)$, any indices $i_1,\ldots,i_T$, and
any $\delta \in (0,1]$, the generalization error for stochastic gradient
descent is at most
\begin{equation}
  R(\est) - \emprisk(\est, S^n)
    \leq 
      4b \sqrt{\frac{\max\Big\{\ceil{\frac{\ln(2Rn)}{\ln (1/\gamma)}} \ln
      2,0\Big\}}{n}}
      + b\sqrt{\frac{8 \ln(2/\delta)}{n}}
      + \frac{2L}{n}
\end{equation}
with probability at least $1-\delta$, where $\gamma = \sqrt{1 - 2 \alpha
\eta + \alpha \beta \eta^2}$.
\end{theorem}
This recovers Theorem~2.1 of
\citet{ParkSimsekliErdogdu2022generalization} while shaving off a $\log
n$ factor. The argument can be extended to losses that are only
piecewise strongly convex and smooth to obtain an analog of their second
main result, \citet[Theorem~2.2]{ParkSimsekliErdogdu2022generalization},
but we omit the details. We further remark that Theorem~\ref{thm:sgd} does not
require that $T$ should be large enough for $\hat \theta$ to be a good
approximation of the empirical risk minimizer (ERM) $\bar \theta =
\argmin_{\theta \in \Theta} \emprisk(\theta, S^n)$. However, if this is
the case, then strong convexity and Lipschitzness of the loss would
imply a better bound on the generalization error, because the ERM is
$\frac{2L^2}{\alpha n}$-uniformly stable and consequently satisfies
$R(\bar \theta) - \emprisk(\bar \theta, S^n) \leq \frac{4L^2}{\delta
\alpha n}$ \citep[Theorem~5]{ShalevShwartzEtAl2010}.
 
\begin{proof}
  Since the losses are $\alpha$-strongly convex and $\beta$-smooth, and
  $\eta \in (0,\beta/2)$, the operator $\Phi_i$ is $\gamma$-contractive
  \citep[Lemma~A.4]{ParkSimsekliErdogdu2022generalization}. This implies
  that for any parameters $\theta,\theta' \in \Theta$, applying the
  last $m \leq T$ iterations of SGD results in points that are close
  together. To formalize this, let $\Phi^m = \Phi_{i_T} \circ \cdots
  \circ \Phi_{i_{T-m+1}}$. Then
  \[
  \|\Phi^m(\theta) - \Phi^m(\theta')\|
    \leq \gamma^m \|\theta - \theta'\|
    \leq \gamma^m R.
  \]
  Given $\epsilon > 0$ yet to be determined, take 
  $
    m = \max\Big\{0,\ceil{\frac{\ln \frac{R}{\epsilon}}{\ln
    (1/\gamma)}}\Big\}
  $
  so that this bound is less than $\epsilon$. There are then two cases for
  $m$:
  \begin{itemize}[leftmargin=*]
    \item $m \leq T$: For any $\sigma,\sigma' \in \{-1,+1\}^n$, let
    $(\est_t)_{t=1,\ldots,T+1}$ and $(\est'_t)_{t=1,\ldots,T+1}$ denote
    the iterates of SGD on the corresponding samples $S_\sigma^n$ and
    $S_{\sigma'}^n$. If $\sigma_{i_t} = \sigma'_{i_t}$ for $t =
    T-m+1,\ldots,T$, i.e.\ the last $m$ iterations of SGD are the same,
    then we have $\|\est_{T+1} - \est'_{T+1}\| \leq \epsilon$ by the
    argument above. It follows that the covering number for $\hat
    \Theta^n$ at radius $\epsilon$ is at most
    \[
      \Cover(\estmodel^n,\metric,\epsilon)
      \leq |\{(\sigma_{i_t})_{t=T-m+1,\ldots,T} \,:\, \sigma \in \{-1,+1\}^n\}|
      = 2^m.
    \]
    \item $m > T$: Then the argument above does not apply, but, since
    the output of SGD only depends on the $T$ data points that it
    actually visits,
    \[
      \Cover(\estmodel^n,\metric,\epsilon)
      \leq |\hat \Theta^n|
      \leq |\{(\sigma_{i_t})_{t=1,\ldots,T} \,:\, \sigma \in \{-1,+1\}^n\}|
      = 2^T
      < 2^m.
    \]
  \end{itemize}
  Both cases, therefore, lead to the same upper bound on the covering
  number. We now note that the generalization error $R(\theta) -
  \hat{R}(\theta,S^n)$ does not change if we replace the loss
  $\ell(\theta,z)$ by $\bar{\ell}(\theta,z) := \ell(\theta,z) -
  h(\theta)$,
  and therefore
  we may assume without loss of generality that the loss is Lipschitz
  continuous instead of only weakly Lipschitz continuous (replacing
  $\ell$ by $\bar \ell$ if necessary). It then follows from
  \eqref{optCoverBound} that
  \[
    \Rad(\hat \Theta^n,S_+^n)
      \leq L \epsilon + b \sqrt{\frac{m \ln 2}{n}} 
      = L \epsilon + b \sqrt{\frac{\max\Big\{\ceil{\frac{\ln
      \frac{R}{\epsilon}}{\ln (1/\gamma)}} \ln 2, 0\Big\}}{n}}.
  \]
  The proof is completed by plugging in $\epsilon = 1/(2n)$, and
  combining this with
  Lemma~\ref{keyThm}.
\end{proof}

\begin{remark}\label{rem:hypstability}
  The previous proof can be adapted to work with the larger class $\bar
  \Theta^n$ instead of $\estmodel^n$, as discussed in
  Section~\ref{BoundRad}. This leads to a slightly worse bound on
  $\Cover(\estmodel^n,\metric,\epsilon)$ of $\binom{2n}{m} \leq
  (\tfrac{e2n}{m})^m$ instead of our current $2^m$, and as a result we
  would get $m \ln(\frac{e2n}{m})$ in place of the current $m \ln 2$,
  which is only an $O(\log n)$ factor worse.
\end{remark}

\subsection{Generalization for compression schemes} 
\label{sec:compression}

Using conditional mutual information, \citet{pmlr-v125-steinke20a} show that the generalization error for the output of a $k$-compression scheme can be
upper bounded by a quantity of order $\mathcal{O}(\sqrt{k \log n / n})$. We show this result is easily recovered using the ARC. An algorithm $\cA$ is a \emph{$k$-compression scheme} if $\cA(S^n) =
\cA_2(\cA_1(S^n))$, where $\cA_1 : \samplespace^n \to \samplespace^k$ maps any
sample $S^n$ of size $n$ to a subsample $S^k \subset S^n$ of size $k
\leq n$, and $\cA_2 : \samplespace^k \to \Theta$ deterministically
determines the final output based only on this subsample.
\begin{theorem} Suppose $\cA$ is a $k$-compression scheme and losses
take values in $[0,1]$. Then, for any $S_-^n$ and $S_+^n$,
  \[ \RadL(\hat \Theta^n,S_+^n) \leq \sqrt{\frac{k\log \frac{2en}{k}}{2 n}}
  = \mathcal{O}\Big(\sqrt{\frac{k \log n}{n}}\Big). \] \end{theorem}

\begin{proof}
Note that the total number of subsamples of length $k$ from $\bigcup_{\sigma \in \{-1,+1\}^n} S_\sigma^n$ is $\binom{2n}{k}$, the number of subsamples of length $k$ from $S_-^n\cup S_+^n$. Consequently, there are at most $\binom{2n}{k}$ possible parameters in $\estmodel^n$:
   \[
  |\hat \Theta^n|
  = |\{ \cA(S_\sigma^n) \, : \, \sigma \in \{-1,+1\}^n\}\}|
  \leq |\{ \cA_1(S_\sigma^n) \, : \, \sigma \in \{-1,+1\}^n\} |
  \leq \binom{2n}{k}.
  \]
  By Massart's lemma \citep{shalev2014understanding},
  $\RadL(\hat \Theta^n,S_+)
      \leq \sqrt{\frac{\log |\hat \Theta^n|}{2 n}}
      \leq \sqrt{\frac{\log \binom{2n}{k}}{2 n}}
      \leq \sqrt{\frac{k}{2 n} \log \frac{2en}{k}}$,
  as required.
\end{proof}

\subsection{Generalization for VC classes}
\label{sec:vc}

For binary classification, \citet{pmlr-v125-steinke20a} further show
that, if $\Theta$ indexes a class of finite VC dimension $V$, then there
exists a version of the empirical risk minimizer (ERM) over that class
for which the conditional mutual information is bounded by
$\mathcal{O}(V \log n)$, leading to a bound on the generalization error
of $\mathcal{O}(\sqrt{V \log n/n})$. An analogous
result, which works for any version of the ERM, can trivially be
obtained from the  ARC, because 
\[
  \RadL(\hat \Theta^n,S_+^n) \leq \RadL(\Theta,S_+^n) = \mathcal{O}\left(\sqrt{\frac{V \log n}{n}}\right),
\]
where the first inequality follows from $\hat \Theta^n \subset \Theta$
and the second is a standard result for Rademacher complexity, obtained
by bounding the Rademacher complexity using the growth function, which
is then controlled using Sauer's lemma \citep{shalev2014understanding}.

 \section{Conclusion and Future Work}
\label{sec:conc}

In this work, we considered algorithm-dependent Rademacher complexity as
an approach to obtain algorithm-dependent generalization bounds.
Circumventing the information-theoretic route, the proposed complexity
notion on the one hand allowed us to derive and unify several known
results with little effort, such as the generalization bound for SGD; on
the other hand it enabled us to link the generalization error to
topological properties of the learning algorithm using the finite
Minkowski dimension \citep{alonso2015hausdorff}. 

We believe that our work opens up several interesting future research
directions. Given the conceptual similarities in conditioning on a
supersample between ARC and CMI, it is natural to ask how the two
complexity measures compare in general. In which cases is one preferable
to the other? Another important connection would be the relation with
algorithmic stability, which is known to play a fundamental role in
generalization with a uniform rate of convergence over all distributions
\citep{ShalevShwartzEtAl2010}.

The concept of finite fractal dimensions, introduced by \cite{alonso2015hausdorff}, turned out to be a fruitful tool to provide bounds with respect to interpretable topological properties without a mutual information term. We believe this is a promising direction, specifically when measuring the generalization bound with respect to a finite hypothesis class.  
It would be interesting to understand whether the definitions of a finite Minkowski dimension can be further adapted for the specific needs in generalization bounds. In particular, it would be interesting to understand if it is possible to relax the definition of a 2-cover further while preserving the limiting behavior.

\paragraph{Acknowledgements} Sachs and Van Erven were supported by the Netherlands Organization for Scientific Research (NWO) under grant number VI.Vidi.192.095. \c{S}im\c{s}ekli is partially supported by the French government under management of Agence Nationale de la Recherche as part of the ``Investissements d’avenir'' program, reference ANR-19-P3IA0001 (PRAIRIE 3IA Institute) and by the European Research Council Starting Grant DYNASTY – 101039676. Part of the research was performed while Sachs was visiting INRIA Paris within the scope of the DYNASTY grant.
\bibliographystyle{apalike}
 \bibliography{literatur}

\appendix

\section{Proofs from Section
\ref{BoundRad}}\label{app:crc_proof}

\begin{proof}[Lemma~\ref{keyThm}]
In order to lighten notation, within this proof we do not write
superscripts $n$, so that $S_-^n = S_-$, $\estmodel^n = \estmodel$, etc.
We further identify $S$ with $S_+$ and we abbreviate
$\emprisk_\xi(\theta) \equiv \emprisk(\theta,S_\xi)$ for any $\xi$.

A common part to both results in the theorem is the following:
\begin{align}
  \ExpD{\sigma}{\max_{\theta \in \estmodel}
  \big\{\emprisk_{-\sigma}(\theta) -
  \emprisk_{\sigma}(\theta)\big\}}
   &=
  \frac{1}{n}\ExpD{\sigma}{\max_{\theta \in \estmodel} 
  \sum_{i=1}^n \sigma_i \big(\loss(Z_i^{-1},\theta) -
  \loss(Z_i^{+1},\theta)\big)} \notag\\
   &\leq
  \frac{1}{n}\ExpD{\sigma}{\max_{\theta \in \estmodel} 
  \sum_{i=1}^n \sigma_i \loss(Z_i^{-1},\theta) + \max_{\theta \in
  \estmodel} \sum_{i=1}^n-\sigma_i \loss(Z_i^{+1},\theta)} \notag\\
   &=
  \RadL(\estmodel,S_-) + \RadL(\estmodel,S_+).
  \label{eqn:common-part}
\end{align}
We now start by proving the in-expectation result:
\begin{align*}
  \ExpD{S_+}{R(\est(S_+)) - \emprisk_+(\est(S_+))} 
   &= 
  \ExpD{S_-,S_+}{\emprisk_-(\est(S_+)) - \emprisk_+(\est(S_+))}\\
   &=
  \ExpD{S_-,S_+,\sigma}{\emprisk_{-\sigma}(\est(S_\sigma)) - \emprisk_\sigma(\est(S_\sigma))}\\
   &\leq
  \ExpD{S_-,S_+,\sigma}{\max_{\theta \in \estmodel}
  \big\{\emprisk_{-\sigma}(\theta) - \emprisk_\sigma(\theta)\big\}}\\
  &\leq \ExpD{S_-,S_+}{\RadL(\estmodel,S_-) + \RadL(\estmodel,S_+)}\\
   &=
  2\ExpD{S_-,S_+}{\RadL(\estmodel,S_+)},
\end{align*}
where the second inequality follows from \eqref{eqn:common-part}. This
completes the proof of \eqref{eqn:rad-exp}.

We proceed to prove the in-probability result. To this end, let
$\epsilon > 0$ be chosen later, and define the following two events:
\begin{align*}
  \event_1 &= \big\{S_+ :
    R(\est(S_+))-\emprisk_+(\est(S_+)) \geq \epsilon\big\},\\
  \event_2 &= \big\{(S_-,S_+) :
    \emprisk_-(\est(S_+)) - \emprisk_+(\est(S_+)) \geq
    \frac{\epsilon}{2}\big\}.
\end{align*}
Our goal will be to bound $\Pr(\event_1)$ and we will start by showing
that
\begin{equation}\label{eqn:ghostsample}
  \Pr(\event_1) \leq 2 \Pr(\event_2).
\end{equation}
This can be established as follows: Since $R(\est(S_+)) -
\emprisk_+(\est(S_+)) \geq \epsilon$ and
$R(\est(S_+))-\emprisk_-(\est(S_+)) \leq \frac{\epsilon}{2}$ together
imply $\event_2$, we have
\begin{align*}
  \Pr(\event_2)
    &\geq \Pr\Big(
      R(\est(S_+)) - \emprisk_+(\est(S_+)) \geq \epsilon
      \quad \text{and}\quad
      R(\est(S_+))-\emprisk_-(\est(S_+)) \leq \frac{\epsilon}{2}
    \Big)\\
    &= \ExpD{S_-,S_+}{
      \ind{R(\est(S_+)) - \emprisk_+(\est(S_+)) \geq \epsilon}
      \ind{R(\est(S_+))-\emprisk_-(\est(S_+)) \leq \frac{\epsilon}{2}}
      } \\
     &= \ExpD{S_+}{
      \ind{R(\est(S_+)) - \emprisk_+(\est(S_+)) \geq \epsilon}
      \Pr_{S_-}\Big(
      R(\est(S_+))-\emprisk_-(\est(S_+)) \leq
      \frac{\epsilon}{2}\Big)}\\
    &\geq \ExpD{S_+}{
      \ind{R(\est(S_+)) - \emprisk_+(\est(S_+)) \geq \epsilon}
      \times \half}\\
    &= \half \Pr(\event_1),
\end{align*}
where, for the last inequality to hold, we restrict attention to
$\epsilon \geq
b \sqrt{\frac{2\log(2)}{n}}$.\footnote{\citet{AnthonyBartlett1999} relax
this to $\epsilon \geq b\sqrt{2/n}$ using a more involved argument based
on Chebyshev's instead of Hoeffding's inequality, but this provides no
benefit here, because we will use a large enough value of $\epsilon$
anyway.} The last inequality then holds because, for any
fixed~$\theta$,
\[
  \Pr_{S_-}\Big( R(\theta)-\emprisk_-(\theta) \leq
  \frac{\epsilon}{2}\Big) \geq 1/2.
\]
To see this, note that by Hoeffding's inequality the probability of the
event's complement is at most
\begin{align*}
  \Pr_{S_-}\Big( R(\theta)-\emprisk_-(\theta) > \frac{\epsilon}{2}\Big)
  \leq \exp\Big(-\frac{n \epsilon^2}{2b^2}\Big)
  \leq 1/2.
\end{align*}
This completes the proof of \eqref{eqn:ghostsample}.

We proceed to work on the right-hand side of \eqref{eqn:ghostsample} by
rewriting the probability of $\event_2$ as follows:
\begin{align*}
  \Pr(\event_2)
    &= \ExpD{S_-,S_+}{
      \ind{%
        \emprisk_-(\est(S_+)) - \emprisk_+(\est(S_+)) \geq
        \frac{\epsilon}{2}
      }
    }\\
    &= \ExpD{S_-,S_+,\sigma}{
      \ind{%
        \emprisk_{-\sigma}(\est(S_\sigma)) -
        \emprisk_\sigma(\est(S_\sigma)) \geq
        \frac{\epsilon}{2}
      }
    }\\
    &= \ExpD{S_-,S_+}{\Pr_\sigma\Big(
        \emprisk_{-\sigma}(\est(S_\sigma)) -
        \emprisk_\sigma(\est(S_\sigma)) \geq
        \frac{\epsilon}{2}
      \Big)}.
\end{align*}
We now restrict attention to $\epsilon$ that are at least
\begin{equation}\label{eqn:epsilonconstraint}
  \epsilon
    \geq 2 \ExpD{\sigma}{\max_{\theta \in \estmodel}
      \big(\emprisk_{-\sigma}(\theta) - \emprisk_\sigma(\theta)\big)}
      + b\sqrt{\frac{8 \log(2/\delta)}{n}}
    \qquad \text{almost surely,}
\end{equation}
so that
\[
  \Pr_\sigma\Big(
        \emprisk_{-\sigma}(\est(S_\sigma)) -
        \emprisk_\sigma(\est(S_\sigma)) \geq
        \frac{\epsilon}{2}
      \Big)
  \leq 
  \Pr_\sigma\Big(\max_{\theta \in \estmodel}
        \big(\emprisk_{-\sigma}(\theta) -
        \emprisk_\sigma(\theta)\big) \geq
        \frac{\epsilon}{2}
      \Big)
  \leq \frac{\delta}{2},
\]
where the last bound holds by McDiarmid's inequality, which applies
because $\max_{\theta \in \estmodel} \big(\emprisk_{-\sigma}(\theta) -
\emprisk_\sigma(\theta)\big)$ has $\frac{2b}{n}$-bounded differences.
Putting everything together, we have shown that $\Pr(\event_1) \leq
\delta$ for any (non-random) $\epsilon$ that satisfies $\epsilon \geq
b\sqrt{2\log(2)/n}$ and \eqref{eqn:epsilonconstraint}. We will show that
this is the case for
\[
  \epsilon = 
    4\,\essup_{S_-,S_+}\,\{\RadL(\estmodel,S_+)\}
      + b\sqrt{\frac{8 \log(2/\delta)}{n}},
\]
from which the intended result \eqref{eqn:rad-prob} then follows.

The first constraint on $\epsilon$ is easiest, because Rademacher
complexity is always non-negative:
\[
  \epsilon
    \geq b\sqrt{\frac{8 \log(2/\delta)}{n}}
    \geq b\sqrt{\frac{2 \log(2)}{n}}.
\]
It remains to check \eqref{eqn:epsilonconstraint}, which follows from
\eqref{eqn:common-part} by
\[
\ExpD{\sigma}{\max_{\theta \in \estmodel}
      \big(\emprisk_{-\sigma}(\theta) - \emprisk_\sigma(\theta)\big)}
  \leq 
  \RadL(\estmodel,S_-) + \RadL(\estmodel,S_+)
  \leq 
  2\,\essup_{S_-,S_+}\,\{\RadL(\estmodel,S_+)\}
  \quad \text{almost surely.}
\]
This completes the proof.
\end{proof}

\section{Definition and basic properties of finite Minkowski dimension}
\label{appendixDimFM}
\subsection{Complete definition of finite Minkowski dimension}
In this section, we restate the complete definition of the finite Minkowski dimension for the convenience of the reader. In the main part of the paper, we excluded e.g. focal or empty sets via several additional assumptions. These assumptions are not needed when considering the more technical original definition by \cite{alonso2015hausdorff}. 

Recall the definition of a 2-cover from the main part: a family of sets $\twoCover$ is a 2-cover of a finite set $C$ if $\twoCover = \{U_i: i \in \NN\}$ where each $U_i \subseteq C$, $|U_i| \geq 2$, and $C \subseteq \bigcup_{U \in \twoCover} U$. Here we will denote the set of all such $2$-covers for $C$ as $K(C)$. Further, let $\Delta(\twoCover) = \max\{ \Delta(U_i) : U_i \in \twoCover\}$ and $K_\delta(C) = \{\twoCover \in K(C):\Delta(\twoCover) \leq \delta\}$. Now define
\begin{align*}
K^1 (C) = \{\twoCover \in K (C) : \Delta(\twoCover) <\Delta(C) \} \qquad \text{ and } \qquad K^1_\delta (C) =  K^1 (C) \cap K_{\delta}(C).
\end{align*}
Note that in the main part, we used the following result to simplify the notation and definition of the finite Minkowski dimension.
\begin{theorem}\cite[Theorem 2.14]{alonso2015hausdorff}
Let $C$ be finite. Then the following are equivalent:
\begin{enumerate}
\item $C$ has no focal point.
\item $K^1(C) \neq \emptyset$.
\item $\nabla (C) < \Delta(C)$.
\end{enumerate}
\end{theorem} 
Next, define a covering, similar to the box covering for the definition of the Minkowski dimension. 
\begin{definition}\cite[Definition 4.1]{alonso2015hausdorff}
For $\twoCover \in K(C)$, set
\[ B_{\twoCover}^s(C) = \absv{\twoCover}\Delta(\twoCover)^s.\]
For $\delta \geq \nabla(C)$, set
\begin{align*}
B_{\delta}^s(C) = \begin{cases}
\min\left\{B_{\twoCover}^s: \twoCover \in K^1_{\delta}(C)\right\}, \quad &\text{ if } K^1(C) \neq \emptyset\\
\min\left\{B_{\twoCover}^s: \twoCover \in K(C)\right\}, \quad &\text{ if } K^1(C) = \emptyset
\end{cases},
\end{align*}
and 
\[ B^s(C) = \max\left\{B_{\delta}^s(C) : \delta \geq \nabla(C)\right\}.\]
\end{definition}
As in Section 4.1 in \cite{alonso2015hausdorff} the finite Minkowski dimension is defined as follows: 
\begin{definition}[finite Minkowski dimension]
Let $s \in (0,\infty)$ such that
\begin{align} 
\label{sol:fM}
B^s(C) = \Delta(C)^s.
\end{align}
The finite Minkowski dimension of a non-empty set $C$ is
\begin{align*}
\dimfM{C} = \begin{cases}
0 \qquad &\text{if } \absv{C} = 1\\
+\infty \qquad &\text{if  $C$ is focal}\\ 
s \, \text{satisfying (\ref{sol:fM})} & \text{otherwise.} 
\end{cases}
\end{align*}

\end{definition}
In Section \ref{SectionFiniteMinkowskiBound} we used the following theorem to simplify the definition of the finite Minkowski dimension.  
\begin{theorem}\cite[Theorem 4.11]{alonso2015hausdorff}
Let $C$ be non-empty and finite. If $C$ is non-focal then
\begin{align*}
\dimfM{C} = \frac{\log T(C)}{\log \frac{\Delta(C)}{\nabla(C)}}. 
\end{align*}
\end{theorem}
 
\begin{theorem}\cite[Theorem 4.12]{alonso2015hausdorff}
Let $\eta:\cX \rightarrow \cX'$ be $(r,\beta)$-Hölder  continuous and $C\subseteq \cX$ finite. Then $\beta \dimfM{\eta(C)} = \dimfM{C}$. 
\end{theorem}
We used the following result to derive the trivial upper bound in Section \ref{SectionFiniteMinkowskiBound}, Theorem \ref{thm:dimfMBoundMain}. 
\begin{theorem}\cite[Proposition 5.3]{alonso2015hausdorff}
Let $C$ be finite. Then
\[\dimfM{C} \leq \frac{\log\left( \absv{C}-1\right)}{\log \frac{\Delta(C)}{\nabla(C)}}. \]

\end{theorem}

\begin{theorem}\cite[Theorem 7.17]{alonso2015hausdorff}
Let $C \subseteq \RR^n$ be compact, with $\nabla(C) = 0 < \Delta(C)$. Then there exists a sequence of sets $\{F_k\}_{k \in \NN}$ with $F_k \rightarrow C$ and $\lim_{k \rightarrow \infty} \dimfM{F_k} = \dimM{X}$.  

\end{theorem}

\section{Proofs from Section \ref{SectionFiniteMinkowskiBound}}
\label{SteinerPoints}
For this section, we need some additional notation. Since these definitions are only needed for the proofs within this section, we define them locally for better readability of the main section. 
\paragraph{Notation:}  For any set $S \in \cX$, we let $\Conv{S}$ denote its convex hull, formally, $ \Conv{S} $ is the unique minimal convex set such that $S \subseteq  \Conv{S} $.

The proof of (\ref{nooutliers}) follows immediately from Lemma \ref{existOptSP}.  

\begin{lemma}[Existence of Optimal Steiner Points]
  \label{existOptSP}
  Suppose $\estmodel^n \subset \RR^k$ is any finite non-focal set. Then for any $0 < \epsilon < \Delta(\estmodel^n)$, there exists a set $P$ with the following properties.
  \begin{enumerate}
  \item  \label{SPcardOpt} $\tabsv{P} < \tabsv{\estmodel^n}$.
  \item \label{SPcoveringdiameterOpt} $\nabla(\estmodel^n \cup P) = \epsilon$ and $\Delta(\estmodel^n \cup P) = \Delta(\estmodel^n)$.
  \item \label{SPcovereq} $\Cover(\estmodel^n,\metric,\epsilon) = \Cover(\estmodel^n \cup P,\metric,\epsilon)$.
  \end{enumerate}
  \end{lemma}
  \begin{proof}
  Let $S$ be the set corresponding to the covering number
  $\Cover(\estmodel^n,\metric,\epsilon)$, i.e., $S$ contains the centers
  of the $\epsilon$-covers for $\estmodel^n$ and $\tabsv{S} =
  \Cover(\estmodel^n,\metric,\epsilon)$. Recall $\ball{c,r,\metric} = \{
  x : \metric(c,x) \leq r\}$ denotes a ball with center $c$ and radius
  $r$. Now consider $P \subset \cup_{a \in S}
  \ball{a,\epsilon,\metric}$, then \ref{SPcovereq} is satisfied. It
  remains to show that we can choose $P \subset \cup_{a \in S}
  \ball{a,\epsilon,\metric}$ such that \ref{SPcardOpt} and
  \ref{SPcoveringdiameterOpt} are satisfied. By restricting the choice
  of $P$ further to $P \subset \cup_{a \in S} \ball{a,\epsilon,\metric} \cap \Conv{\estmodel^n}$ we can ensure that the diameter $\Delta(\estmodel^n \cup P) = \Delta(\estmodel^n)$. Due to the assumptions on $\epsilon$, the conditions $\tabsv{P} < \tabsv{\estmodel^n}$ and $\nabla(\estmodel^n \cup P) \leq \epsilon$ can also be satisfied by choosing the points in $P$ from the boundaries of the epsilon balls, i.e., $P \subset \cup_{a \in S}\bd{\ball{a,\epsilon,\metric}} \cap \Conv{\estmodel^n}$. 
\end{proof}

  We are now ready to prove Theorem \ref{thm:limit}.

 \begin{proof}[Theorem \ref{thm:limit}]
 Let $\epsilon_n = \alpha \sqrt{\frac{\log n}{n}}$ for any constant
 $\alpha > 0$. Then, by
\eqref{optCoverBound},
\[
  \RadL(\estmodel^n, S_+^n)
  	\leq
        L\epsilon_n+b\sqrt{\frac{\log\Cover(\estmodel^n,\metric,\epsilon_n)}{n}}
  	\leq
        L\epsilon_n+b\sqrt{\frac{\log\Cover(\estmodel,\metric,\epsilon_n)}{n}}.
\]
Hence
\begin{align*}
  \limsup_{n \to \infty} \frac{\RadL(\estmodel^n,
  S_+^n)}{\sqrt{\log(n)/n}}
  &\leq \alpha L + \limsup_{n \to \infty}
   b\sqrt{\frac{\log\Cover(\estmodel,\metric,\epsilon_n)}{\log n}}\\
  &= \alpha L + \limsup_{n \to \infty}
   b\sqrt{\frac{\log\Cover(\estmodel,\metric,\epsilon_n)}{-\log \epsilon_n}}
    \sqrt{\frac{-\log(\epsilon_n)}{\log n}}\\
  &= \alpha L + 
   b\sqrt{\frac{\overline{\dim}_\mathcal{M}(\estmodel)}{2}},
\end{align*}
where the last equality follows by the definition of the upper
Minkowski dimension \eqref{eqn:dimm}
and because
\[
  \frac{-\log(\epsilon_n)}{\log n}
  = 
  \frac{-\log \alpha - \frac12 \log \log n + \frac12 \log n}{\log n}
  \to \frac12
  \qquad \text{as $n \to \infty$.}
\]
The result follows by letting $\alpha$ tend to $0$.
\end{proof}

\end{document}